\documentclass[]{article}
\usepackage{fullpage}
\usepackage[numbers]{natbib}
\usepackage{authblk}
\usepackage{xcolor}

\usepackage{optidef}
\usepackage[ruled,vlined]{algorithm2e}

\usepackage[normalem]{ulem}
\useunder{\uline}{\ul}{}
\usepackage{booktabs}
\usepackage{caption}
\usepackage{subcaption}
\usepackage{amsfonts} 
\usepackage{amsthm}

\newtheorem{definition}{Definition}
\newtheorem{theorem}{Theorem}

\newtheorem{assumption}{Assumption}
\newcommand{\mmFair}{\textsc{MinimaxFair}}
\newcommand{\mmFairCon}{\textsc{MinimaxFairRelaxed}}

\providecommand{\keywords}[1]
{
\noindent
\textbf{\textit{Keywords} --- } #1
}
\def\werm{{\sc WERM$(H)$\xspace}}

\newcommand\shortversion[1]{}

\AtBeginDocument{%
  \providecommand\BibTeX{{%
    \normalfont B\kern-0.5em{\scshape i\kern-0.25em b}\kern-0.8em\TeX}}}

\title{Minimax Group Fairness: Algorithms and Experiments}

\author[1,2]{Emily Diana}
\author[1,2]{Wesley Gill}
\author[1,2]{Michael Kearns}
\author[2]{Krishnaram Kenthapadi}
\author[1,2]{Aaron Roth}
\affil[1]{University of Pennsylvania}
\affil[2]{Amazon AWS AI}

\date{\today}

\begin{document}
\maketitle
\begin{abstract}
We consider a recently introduced framework in which fairness is measured by {\em worst-case\/} outcomes across groups, rather than by the more standard differences between group outcomes. In this framework we provide provably convergent {\em oracle-efficient\/} learning algorithms (or equivalently, reductions to non-fair learning) for {\em minimax group fairness\/}. Here the goal is that of minimizing the maximum loss across all groups, rather than equalizing group losses. Our algorithms apply to both regression and classification settings and support both overall error and false positive or false negative rates as the fairness measure of interest. They also support relaxations of the fairness constraints, thus permitting study of the tradeoff between overall accuracy and minimax fairness. We compare the experimental behavior and performance of our algorithms across a variety of fairness-sensitive data sets and show empirical cases in which minimax fairness is strictly and strongly preferable to equal outcome notions.
\end{abstract}


\section{Introduction}
\label{sec:introduction}
Machine learning researchers and practitioners have often focused on achieving group fairness with respect to protected attributes such as race, gender, or ethnicity. Equality of error rates is one of the most intuitive and well-studied group fairness notions,
and in enforcing it one often implicitly hopes that higher error rates on protected or ``disadvantaged'' groups will be reduced towards the lower error rate of the majority or ``advantaged'' group. 
But in practice, equalizing error rates and similar notions may require artificially inflating error on easier-to-predict groups --- without necessarily decreasing the error for the harder to predict groups --- and this 
may be undesirable for a variety of reasons.

For example, consider the many social applications of machine learning in which most or even all of the targeted
population is disadvantaged, such as predicting domestic situations in which children may be at risk of physical
or emotional harm \cite{chouldechova2018case}. 
While we might be interested in ensuring that our predictions are roughly equally accurate across racial groups, income levels,
or geographic location, if this can only be achieved by raising lower group error rates without lowering the error for any other population, then arguably we will have
only worsened overall social welfare, since this is not a setting where we can argue that we are ``taking from the rich and
giving to the poor.'' Similar arguments can be made in other high-stakes applications, such as predictive modeling
for medical care. In these settings it might be preferable to consider the alternative fairness criterion of achieving 
{\em minimax group error\/}, in which we seek not to equalize error rates, but to minimize the largest group error rate ---
that is, to make sure that {\em the worst-off group is as well-off as possible\/}.

Minimax group fairness, which was recently proposed by \cite{Martinez} in the context of classification, has the property that any model that
achieves it {\em Pareto dominates\/} (at least weakly, and possibly strictly) an equalized-error model with respect
to group error rates --- that is, if $g_i$ is the error rate on group $i$ in a minimax solution, and $g'$ is the 
common error rate for all groups in a solution that equalizes group error rates, then $g' \geq g_i$ for all $i$. If one or more of these inequalities is strict, it constitutes a proof
that equalized errors can only be achieved by deliberately inflating the error of one or more groups in the minimax solution. Put another way, one technique for finding an optimal solution subject to an equality of group error rates constraint is to first find a minimax solution, and then to artifically inflate the error rate on any group that does not saturate the minimax constraint --- an ``optimal algorithm'' that makes plain the deficiencies of equal error solutions.

In contrast to approaches that learn separate models for each protected group --- which also satisfy minimax error, simply by optimizing each group independently --- (e.g. \cite{decoupled1,decoupled2}), the minimax approach we use here has two key advantages:
\begin{itemize}
    \item The minimax approach does not require that groups of interest be disjoint, which is a requirement for the approach of learning a different model for each group. This allows for protecting groups defined by intersectional characteristics  as in \cite{Kearns,subgroup2}, protecting (for example) not just groups defined by race or gender alone, but also by combinations of race and gender.
    \item The minimax approach does not require that protected attributes be given as inputs to the trained model. This can be extremely important in domains (like credit and insurance) in which using protected attributes as features for prediction is illegal. 
\end{itemize}

\noindent Our primary contributions are as follows:
\begin{enumerate}
    \item First, we propose two algorithms. The first finds a minimax group fair model from a given statistical class, and the second navigates tradeoffs between a relaxed notion of minimax fairness and overall accuracy.
    \item Second, we prove that both algorithms converge and are oracle efficient --- meaning that they can be viewed as efficient reductions to the problem of unconstrained (nonfair) learning over the same class. We also study their generalization properties.
    \item Third, we show how our framework can be easily extended to incorporate different types of error rates -- such as false positive and negative rates -- and explain how to handle overlapping groups.
    \item Finally, we provide a thorough experimental analysis of our two algorithms under different prediction regimes. In this section, we focus on the following:
    \begin{itemize}
        \item We start with a demonstration of the learning process of  learning a fair predictor from the class of linear regression models. This setting matches our theory exactly, because weighted least squares regression is a convex problem, and so we really do have efficient subroutines for the unconstrained (nonfair) learning problem.
        \item We conduct an exploration of the fairness vs. accuracy tradeoff for regression and highlight an example in which our minimax algorithms provide a substantial Pareto improvement over the equality of error rates notion.
        \item Next, we give an account of the difficulties encountered when our oracle assumption fails in the classification case (because there are no efficient algorithms for minimizing 0/1 classification error, and so we must rely on learning heuristics).
        \item With this in mind, we again explore tradeoff curves for the classification case and finish with another comparison in which we show marked improvement over equality of error rates.
    \end{itemize}
\end{enumerate}

\subsection{Related Work}
There are similar algorithms known for optimizing minimax error in other contexts, such as scheduling, fair division, and clustering --- see e.g. \cite{scheduling, fairdivision, samadi2018price, chen, cotter}). Particularly related are \cite{chen} and \cite{cotter}; \cite{chen} employs a method similar to our first algorithm to minimize a set of non-convex losses. In applying similar techniques to a fair classification domain, our conceptual contributions include the ability to handle overlapping groups.

Our work builds upon the notion of minimax fairness proposed by Martinez et al. \cite{Martinez} in the context of classification. Their algorithm -- Approximate Projection onto Start Sets (APStar) -- is also an iterative process which alternates between finding a model that minimizes risk and updating group weightings using a variant of gradient descent, but they provide no convergence analysis. More importantly, APStar relies on knowledge of the base distributions of the samples (while our algorithms do not) and it does not provide the capability to relax the minimax constraint and explore an error vs. fairness tradeoff. \cite{Martinez} analyze only the classification setting, but we provide theory and perform experiments in both classification and regression settings. Because our meta-algorithm is easily extensible, we are able to generalize to non-disjoint groups and various error types (with an emphasis on false positive and false negative rates). 

Achieving minimax fairness over a large number of groups has been proposed by \cite{lahoti2020fairness} as a technique for achieving fairness when protected group labels are not available. Our work relates to \cite{lahoti2020fairness} in much the same way as it relates to \cite{Martinez}, in that \cite{lahoti2020fairness} is purely empirical, whereas we provide a formal analysis of our algorithm, as well as a version of the algorithm that allows us to relax the minimax constraint and explore performance tradeoffs. 

Technically, our work uses  a similar approach to  \cite{agarwal2018reductions, agarwal2019fair}, which also reduce a ``fair'' classification problem
to a sequence of unconstrained cost-sensitive classification
problems via the introduction of a zero-sum game. Crucial to this line of work is the connection between no regret learning and equilibrium convergence, originally due to \cite{Freund}.

\section{Framework and Preliminaries}\label{sec:prelim}
We consider pairs of feature vectors and labels ${(x_i,y_i)}_{i=1}^{n}$, where $x_i$ is a feature vector, divided into $K$ groups $\{G_1,\ldots,G_K\}$.  We choose a class $H$ of models (either classification or regression models), mapping features to predicted labels, with loss function $L$ taking values in $[0,1]$, average population error: $$\epsilon(h) = \frac{1}{n} \sum_{i=1}^{n} L(h(x_i),y_i)$$ and average group loss: $$\epsilon_k(h) = \frac{1}{|G_k|}\sum_{(x,y) \in G_k} L(h(x),y)$$ for some $h \in H$. We also admit randomized models in this paper, which can be viewed as belonging to the set $\Delta H$ of probability distributions over $H$. 
We define population loss and group loss for a distribution over models as the expected loss over a random choice of model from the distribution.   

First, in the pure minimax problem, our goal is to find a randomized model $h^*$ that minimizes the maximum error rate over all groups: 
\begin{align}
\label{eqn:minimax}
 h^* = \text{argmin}_{h \in \Delta H} \left\{\max_{1 \le k \le K} \epsilon_k(h)\right\}
\end{align} 
We let $\mathrm{OPT}_1$ denote the value of the solution to the minimax problem: $\mathrm{OPT}_1 = \max_k \epsilon_k(h^*).$ We say that a randomized model $h$ is $\epsilon$-approximately optimal for the minimax objective if:
\[\max_k \epsilon_k(h) \leq \mathrm{OPT}_1+\epsilon\]

We next describe an extension of the minimax problem:  Given a target maximum group error bound $\gamma \geq \mathrm{OPT}_1$, the goal is to find a randomized model that minimizes overall population error while staying below the specified maximum group error threshold:

\begin{mini}
    {h \in \Delta H}{\epsilon(h)}{}{}
    \addConstraint{\epsilon_k(h)}{\leq \gamma,\:}{ k=1,\ldots,K}
\label{eqn:constrained}
\end{mini}

This extension has two desirable properties:
\begin{enumerate}
    \item It has an objective function: there may in principle be many minimax optimal models that have different overall error rates. The constrained optimization problem defined above breaks ties so as to select the model with lowest overall error. 
    \item The constrained optimization problem allows us to trade off our maximum error bound $\gamma$ with our overall error, rather than requiring us to find an exactly minimax optimal model. In some cases, this tradeoff may be worthwhile in that small increases in $\gamma$ can lead to large decreases in overall error. 
\end{enumerate}

Given a maximum error bound $\gamma$, we write $\mathrm{OPT}_2$ for the optimum value of Problem~(\ref{eqn:constrained}). We say that a randomized model $h$ is an $\epsilon$-approximate solution to the constrained optimization problem in (\ref{eqn:constrained}) if $\epsilon(h) \leq \mathrm{OPT}_2 + \epsilon$, \emph{and} for all $k$, $\epsilon_k(h) \leq \gamma + \epsilon$.

In order to solve Problems~(\ref{eqn:minimax}) and ~(\ref{eqn:constrained}), we  pose the problems as two player games in Section~\ref{sec:game}. This will rely on several classical concepts from game theory, which we will expand upon in Section~\ref{sec:theory}. We also define a weighted empirical risk minimization oracle over the class $H$, which we will use as an efficient non-fair subroutine in our algorithms.

\begin{definition}[Weighted Empirical Risk Minimization Oracle]
\label{def:werm}
A weighted empirical risk minimization oracle for a class $H$ (abbreviated \werm) takes as input a set of $n$ tuples $(x_i,y_i)_{i=1}^{n}$, a weighting of points $w$, and a loss function $L$, and finds a hypothesis $\hat{h} \in H$ that minimizes the weighted loss, i.e., $\hat{h} \in \text{argmin}_{h \in H} \sum_{i=1}^{n}w_i L(h(x_i),y_i)$.
\end{definition}

\section{Two Player Game Formulation}
\label{sec:game}
Starting with Problem~(\ref{eqn:minimax}), we recast the optimization problem as a zero-sum game between a learner and a regulator in \mmFair. At each round $t$, there is a weighting over groups determined by the regulator. The learner (best) responds by computing model $h_t$ to minimize the weighted prediction error. The regulator updates group weights using the well-known Exponential Weights algorithm 
% \mk{capitalize algo name, and someone add citation}\emily{Done} 
with respect to group errors achieved by $h_t$ \cite{cesa-bianchi_lugosi_2006}. The learner's final model $M$ is the randomized model with uniform weights over the $h_t$’s produced. In the limit, $M$ converges to a minimax solution with respect to group error. In particular, over $T$ rounds, with $\eta_t \approx 1 /\sqrt{t}$, the empirical average of play forms a $1 / \sqrt{T}$-approximate Nash equilibrium \cite{Freund}.

\begin{algorithm}
\SetAlgoLined
\KwIn{ $\{x_i,y_i\}_{i=1}^{n}$, adaptive learning rate $\eta_t$, populations $G_k$ with relative sizes $p_k=\frac{|G_k|}{n}$, iteration count $T$, loss function $L$, model class $H$\;
 }
Let $\epsilon_k(h) := \frac{1}{|G_k|}\sum_{(x,y) \in G_k} L(h(x),y)$\;
Initialize $\lambda_k := p_k\; \forall k$\;
 \For{$t=1$ \textbf{to} $T$}{
  Find $h_t := \text{argmin}_{h \in H} \sum_k \lambda_k \cdot \epsilon_k(h)$\;
  Update each $\lambda_k := \lambda_k \cdot exp(\eta_t \cdot \epsilon_k(h_t))$\;
 }
\KwOut{Uniform distribution over the set of models ${h_1,...,h_T}$}
\caption{\mmFair}
\label{alg:game}
\end{algorithm}

As discussed in Section~\ref{sec:prelim}, not all minimax solutions achieve the same overall error. By setting an acceptable maximal group error $\gamma$, we can potentially lower overall error by solving the relaxed version in Problem~(\ref{eqn:constrained}). 

Letting $p_k=\frac{|G_k|}{n}$ be group proportions, and assuming that the groups are disjoint here for simplicity,\footnote{The derivation for overlapping groups is given in the Appendix.} the Lagrangian dual function of Problem~(\ref{eqn:constrained}) is given by:

\begin{align*}
F(\lambda,h) &= \epsilon(h) + \sum_k \lambda_k(\epsilon_k(h) - \gamma) \\
&= \sum_k \left( (p_k + \lambda_k) \epsilon_k(h) - \lambda_k \gamma \right)
\end{align*}

We again cast this problem as a game in \mmFairCon\ where the learner chooses $h$ to minimize $\sum_k (p_k + \lambda_k) \epsilon_k(h)$, and the regulator adjusts $\mathbb{\lambda}$ through gradient ascent with gradient $\frac{\delta F}{\delta \lambda_k} = \epsilon_k(h) - \gamma$.
As before, the empirical average of play converges to a Nash equilibrium, where an equilibrium corresponds to an optimal solution to the original constrained optimization problem.
    
\begin{algorithm}
\SetAlgoLined
\KwIn{ $\{x_i,y_i\}_{i=1}^{n}$, adaptive learning rate $\eta_t$, populations $G_k$ with relative sizes $p_k=\frac{|G_k|}{n}$, iteration count $T$, loss function $L$, model class $H$, maximal group error $\gamma$\;
 }
Let $\epsilon_k(h) := \frac{1}{|G_k|}\sum_{(x,y) \in G_k} L(h(x),y)$\;
Initialize $\lambda_k := 0\; \forall k$ \;
 \For{$t=1$ \textbf{to} $T$}{
  Find $h_t := \text{argmin}_{h \in H} \sum_k (p_k + \lambda_k) \cdot \epsilon_k(h)$\;
  Update each $\lambda_k := \max{(\lambda_k + \eta_t \cdot (\epsilon_k(h_t) - \gamma), 0)}$\;
 }
\KwOut{Uniform distribution over the set of models ${h_1,...,h_T}$}
\caption{\mmFairCon}
\label{alg:constrained}
\end{algorithm} 

\section{Theoretical Guarantees}
\label{sec:theory}
In this section we state the theoretical guarantees of our algorithms. To do so, we make an assumption, now standard in the fair machine learning literature, which allows us to bound the \emph{additional} hardness of our fairness desiderata, on top of nonfair learning:

\begin{assumption}
\label{oracle assumption}
(Oracle Efficiency)
We assume the learner has access to a weighted empirical risk minimization oracle over the class $H$, \werm, as specified in Definition~\ref{def:werm}.
\end{assumption}

This assumption will be realized in practice whenever the objective is convex (for example, least squares linear regression). When the objective is not convex (for example, 0/1 classification error) we will employ heuristics in our experiments which are not in fact oracles, resulting in a gap between theory and practice that we investigate empirically. 

\subsection{\mmFair}

\begin{theorem}
After $T = \frac{\ln K}{2 \epsilon^2}$ many rounds, \mmFair\ returns a randomized hypothesis that is an $\epsilon$-optimal solution to Problem~(\ref{eqn:minimax}).
\end{theorem}

\begin{proof}
From \cite{cesa-bianchi1999, cesa-bianchi_lugosi_2006,10.1006/jcss.1997.1504}, under the conditions of Assumption 1, with loss function $L(\cdot)$, $K$ groups, $T$ time steps, and step size $\eta = \frac{8 \ln K}{T}$ the Exponential Weights update rule of the regulator yields regret:
%\emily{Specify what regret means in this context}
\begin{align}
\frac{R_T}{T} &= \frac{1}{T} \left(\sum_{t=1}^{T}L(h_t(x),y) - \min_{h_i \in H} \sum_{i=1}^{T} L(h_i(x),y)\right) \\
&\leq \sqrt{\frac{\ln K}{2T}}
\end{align}

Plugging in $T = \frac{\ln K}{2 \epsilon^2}$ gives
$\frac{R_T}{T} \leq \sqrt{\frac{\ln K}{2\frac{\ln K}{2 \epsilon^2} }} = \epsilon$. 

As the learner plays a \textit{best-response} strategy by calling \werm, applying the following result of \cite{Freund} completes the proof:

\begin{theorem}[Freund and Schapire, 1996]
\label{thm:freund}
Let $h_1,...h_T$ be the learner's sequence of models and $\lambda_1,...,\lambda_T$ be the regulator's sequence of weights. Let $\bar{h} = \frac{1}{T} \sum_{i=1}^{T} h_t$ and $\bar{\lambda} = \frac{1}{T} \sum_{i=1}^{T} \lambda_i$. Then, if the regret of the regulator satisfies $\frac{R_G(T)}{T} \leq \epsilon$
and the learner best responds in each round, ($\bar{h},\bar{\lambda}$) is an $\epsilon$-approximate solution.
\end{theorem}

Therefore, the uniform distribution over $h_1,...,h_T$ obtained by the learner in \mmFair\ is an $\epsilon$-optimal solution to Problem~(\ref{eqn:minimax}), as desired. We note that this requires one call to \werm\ for each of the $T$ rounds.

\end{proof}

\subsection{\mmFairCon}

\begin{theorem}
\label{thm:alg2}
After $T=\frac{1}{4\epsilon^2}\left(\frac{1}{\epsilon^2} + 2K\right)^2$ many rounds, \mmFairCon\ returns a randomized hypothesis that is an $\epsilon$-optimal solution to Problem~(\ref{eqn:constrained}).
\end{theorem}

\begin{proof}

In \mmFairCon\ the regulator plays a different strategy: Online Gradient Descent. We specify the regret bound from \cite{zinkevich} below. Note that in the following definition, $F$ is the set containing all values of $\lambda$ (the vector updated through the gradient descent procedure), and the size of the set $F$ is denoted as:
$$||F|| = \max_{\lambda_1,\lambda_2 \in F} d(\lambda_1,\lambda_2) = \max_{\lambda_1,\lambda_2 \in F} ||\lambda_1-\lambda_2||$$

In our analysis, we compute our regret compared to the best vector of  weights such that $||\lambda|| \leq \frac{1}{\epsilon} = ||F||$. We write $||\nabla c||$ for the norm of the gradients that we feed to gradient descent. As our losses are bounded by $[0,1]$, $||\nabla{c}||^2 = \sum_{k=1}^K (\epsilon_k(h_t) - \gamma)^2 \leq K$. Then, from \cite{zinkevich}, with $\eta_t=t^{-\frac{1}{2}}$, 
$$R_T \leq \frac{||F||^2 \sqrt{T}}{2}+\left( \sqrt{T} - \frac{1}{2}\right)|| \nabla c||^2$$

Plugging in the specification for our problem, we have
\begin{align*}
\frac{R_T}{T} &\leq \frac{1}{T}\left(\frac{\frac{1}{\epsilon^2} \sqrt{T}}{2}+K\left( \sqrt{T} - \frac{1}{2}\right)\right)
< \frac{\frac{1}{\epsilon^2} + 2K}{2\sqrt{T}}
\end{align*}

Substituting $T=\frac{1}{4\epsilon^2}\left(\frac{1}{\epsilon^2} + 2K\right)^2$ yields

$$\frac{R_T}{T} \leq \frac{\frac{1}{\epsilon^2} + 2K}{2\sqrt{\frac{1}{4\epsilon^2}\left(\frac{1}{\epsilon^2} + 2K\right)^2}}=\epsilon$$

Therefore, Theorem~\ref{thm:freund} guarantees that the value of the objective is not more than $\epsilon$ away from $OPT_2$, using the notation of Section~\ref{sec:prelim}.

Finally, we show that the learner cannot choose a model that violates a constraint by more than $\epsilon$. Suppose the learner chose a randomized strategy $\tilde{h}$ such that $\epsilon_k(\tilde{h}) > \gamma + \epsilon$ for some $k$. Then, the regulator may set $\lambda_k=\frac{1}{\epsilon}$ and $\lambda_j=0$ for all $j \neq k$, yielding 
$$\max_\lambda F(\lambda,\tilde{h}) > \epsilon(\tilde{h}) + \epsilon \cdot \frac{1}{\epsilon}= \epsilon(\tilde{h}) + 1$$

However, because $\epsilon(\tilde{h}) \leq 1$ by Assumption 1, the learner is better off selecting $h'$ such that $\epsilon(h') = \gamma$ for all groups, even if $\epsilon(h')=1$. Then, $\max_\lambda F(\lambda,h') = \epsilon(h') \leq 1$.

Thus, an $\epsilon$-approximate equilibrium distribution $h$ for the  learner must satisfy $\epsilon_k(h) \leq \gamma + \epsilon$ to minimize the value of $\max_\lambda F(\lambda,h)$. We have shown that in $T=\frac{1}{4\epsilon^2}\left(\frac{1}{\epsilon^2} + 2K\right)^2$ rounds -- with one call to \werm\ per round -- \mmFairCon\ always outputs a model $h$ such that $\epsilon(h) \leq OPT_2 + \epsilon$ and for all $k$, $\epsilon_k(h) \leq \gamma + \epsilon$, or an $\epsilon$-optimal solution to Problem~(\ref{eqn:constrained}).

\end{proof}
\subsection{Generalization}
\label{subsec:generalization}
Finally, we analyze the generalization properties of \mmFair\ and \mmFairCon. We observe that if we have uniform convergence over group errors for each deterministic model in $H$, then we also achieve uniform convergence over each group with randomized models in $H$. Because this is a uniform convergence statement, it does not interact with the weightings of our algorithms. The proof of the following theorem can be found in the Appendix.

\begin{theorem}
\label{thm:generalization}

Fix $\delta>0$, let $d$ be the VC dimension of the class $H$, and let $n_1,\ldots,n_K$ be the sample sizes of groups $i=1,\ldots,K$ in sample $S$ drawn from distribution $D$.
Recall that $\epsilon_i(h)$ denotes the error rate of $h$ on the $n_i$ samples of group $i$ in S, and let
$\epsilon'_i(h)$ denote the expected error of $h$ with respect to $D$ conditioned on group $i$.

Then for every $h \in H$, and every group $i$, with probability at least $1-\delta$ over the randomness of $S$:
$$|{\epsilon_i}(h) - \epsilon'_i(h) | \leq O\left( \sqrt{\frac{\log{\frac{1}{\delta}} + d \log n_i}{n_i}}\right)$$
\end{theorem}
Note that this bounds the generalization gap \emph{per group}. An immediate consequence is that, with high probability, the gap between the in and out-of-sample minimax error can be bounded by $O\left(\max_i  \sqrt{\frac{\log{\frac{K}{\delta}} + d \log n_i}{n_i}}\right)$, a quantity that is dominated by the sample size of the \emph{smallest} group. 

\section{Extension to False Positive/Negative Rates}
\label{sec:falsepos}
A  strength of our framework is its generality. With minor alterations, \mmFairCon\ can also be used to bound false negative or false positive group error in classification settings while minimizing overall population error.\footnote{Note that a minimax false positive (negative) rate of zero can be achieved by always predicting negative (positive). Therefore, it is trivial to solve the exact minimax problem for false positive or false negative rates, but the relaxed problem is nontrivial.}

To extend \mmFairCon\ to bound false positive rates, we again consider a setting in which we are given groups $\{G_k\}_{k=1}^{K}$ containing points ${(x_i,y_i)}_{i=1}^{n}$. We will want to bound the error, here denoted by $\epsilon(h,(x_i,y_i)) = |h(x_i)-y_i|$, on the parts of each group that contain true negatives: $G_k^0 = \{(x, y) : (x, y) \in G_k \text{ and } y = 0\}$. We want to minimize overall population error while keeping all group false positive rates below $\gamma$, and our adapted constrained optimization problem is:
\begin{mini}
    {h \in H}{\frac{1}{n}\sum_{i=1}^{n} \epsilon(h,(x_i,y_i))}{}{}
    \addConstraint{ \frac{1}{|G_k^0|} \sum_{(x_i,y_i) \in G_k^0} \epsilon(h, (x_i,y_i))}{\leq \gamma,\:}{ k=1,\ldots,K}
\label{eqn:falsePos}
\end{mini}

The Lagrangian dual for Problem~(\ref{eqn:falsePos}) is:
\begin{align*}
    &F(\lambda,h) = \frac{1}{n}\sum_i \epsilon(h,(x_i,y_i)) \\
    &+ \sum_k \lambda_k \left( \frac{1}{|G_k^0|} \sum_{(x_i,y_i) \in G_k^0} \epsilon(h, (x_i,y_i)) - \gamma \right) \\
    &= \sum_i \epsilon(h,(x_i,y_i)) \left( \frac{1}{n} + \sum_k \frac{\lambda_k}{|G_k^0|}\mathbb{I}\{(x_i,y_i) \in G_k^0\} \right) - \sum_k \lambda_k \gamma \\
    &= \sum_i w_i \epsilon(h,(x_i,y_i)) - \sum_k \lambda_k \gamma 
\end{align*}
where $w_i = \frac{1}{n} + \sum_k \frac{\lambda_k}{|G_k^0|}\mathbb{I}\{(x_i,y_i) \in G_k^0\}$ and $$\frac{\delta F}{\delta \lambda_k} = \frac{1}{|G_k^0|} \sum_i \mathbb{I}\{(x_i,y_i) \in G_k^0\} \epsilon(h, (x_i,y_i)) - \gamma$$

We can then use the sample weights $w$ in the learner's step of the constrained optimization problem. In particular, at each round $t$, the learner will find $h_t := \text{argmin}_{h \in H} \sum_i w_i \cdot \epsilon(h,(x_i,y_i))$
and the regulator will update the $\lambda$ vector with

$$\lambda_k := \left(\lambda_k + \eta_t \cdot \left(\frac{1}{|G_k^0|} \sum_i \mathbb{I}\{(x_i,y_i) \in G_k^0\} \epsilon(h, (x_i,y_i)) - \gamma\right)  \right)_+$$

\section{Experimental Results}
\label{sec:experiments}
We experiment with our algorithms in both regression and classification domains using real data sets. A brief summary of our findings and contributions:

\begin{itemize}
    \item Our minimax algorithm often admits solutions with significant Pareto improvements over equality of errors in practice.
    \item Unlike similar equal error algorithms, our minimax and relaxed algorithms use only non-negative sample weights, increasing performance for both regression and classification via access to better base classifiers, which can optimize arbitrary convex surrogate loss functions (because non-negative weights preserve the convexity of convex loss functions, whereas negative weights do not).

    \item We illustrate the tradeoff between overall error and fairness by explicitly tracing the Pareto curve of possible models between the population error minimizing model and the model achieving (exact) minimax group fairness. We extend this to the case of false positive or false negative group fairness by using our relaxed algorithm over a range of $\gamma$ values.
    
    \item We demonstrate strong generalization performance of the models produced by our algorithms on datasets with sufficiently sized groups 
    
\end{itemize}

\subsection{Methodology and Data}
\label{ssec:methodologydata}
\subsubsection{Data}
\label{data}
We begin with a short description of each of the data sets used in our experiments: 

\begin{itemize}
    \item \textbf{Communities and Crime~\cite{REDMOND2002660, census1, census2, census3, census4}:} Communities within the US. The data combines socio-economic data from the 1990 US Census, law enforcement data from the 1990 US LEMAS survey, and crime data from the 1995 FBI UCR.

    \item \textbf{Bike~\cite{E2020353, doi:10.1080/22797254.2020.1725789, Dua:2019}:} Public bikes rented at each hour in Seoul Bike sharing System with the corresponding weather data and holidays information.

    \item\textbf{COMPAS~\cite{compas}:} Arrest data from Broward County, Florida, originally compiled by ProPublica.

    \item\textbf{Marketing~\cite{Moro, Dua:2019}: } Data related with direct marketing campaigns (phone calls) of a Portuguese banking institution.
    
    \item\textbf{Student~\cite{Cortez2008UsingDM, Dua:2019}: } Student achievement in secondary education of two Portuguese schools.
    
    \item \textbf{Internet Traffic \cite{kdd, Dua:2019}: } Data set of network connection data for training to distinguish between `bad' connections, called intrusions or attacks, and `good' normal connections. We subsampled the larger groups in this dataset to 10\% of their original size for experimental feasibility. 

\end{itemize}
The table below outlines the data sets used in our experiments. For all data sets, categorical features were converted into one-hot encoded vectors and group labels were included in the feature set. 
\begin{table}[ht]
\centering
\scriptsize
\begin{tabular}{|l|c|c|l|l|l|}
\hline
Dataset   & \multicolumn{1}{l|}{$n$} & \multicolumn{1}{l|}{$d$} & Label                                                                       & Group     & Task           \\ \hline
\begin{tabular}[c]{@{}l@{}}Communities\\  and Crime\end{tabular} & 1594 & 133 & \begin{tabular}[c]{@{}l@{}}violent crimes\\  per pop.\end{tabular} & race & regression \\ \hline
Bike     & 8760                     & 19                       & \begin{tabular}[c]{@{}l@{}}rented bikes\\ normalized\end{tabular}           & season    & regression     \\ \hline
COMPAS    & 4904                     & 9                        & \begin{tabular}[c]{@{}l@{}}two year\\ recidivism\end{tabular}               & race, sex & classification \\ \hline
Marketing & 45211                    & 48                       & \begin{tabular}[c]{@{}l@{}}subscribes\\ to term deposit\end{tabular} & job       & classification \\ \hline
Student & 395 & 75 & final grade & sex & classification \\ \hline
Internet Traffic & 50790  & 112  & \begin{tabular}[c]{@{}l@{}}connection\\ legitimacy\end{tabular}& \begin{tabular}[c]{@{}l@{}}protocol\\ type \end{tabular} & classification \\ \hline
\end{tabular}
\label{tab:data}
\end{table}
\subsubsection{Train/Test Methodology}
\label{training-data}
We have two goals in our experiments: to illustrate the \emph{optimization} performance of our algorithms, and to illustrate the \emph{generalization} performance of our algorithms. Our first set of experiments aims to illustrate optimization --- and tradeoffs between overall accuracy and upper bounds on individual group error rates. We perform this set of experiments on the (generally quite small) fairness-relevant datasets, and plot \emph{in-sample} results. 

Our next set of experiments demonstrate the generalization performance of our algorithms on large, real datasets\footnote{We remark that while some of 
these larger datasets, such as the one on Internet traffic, have no obvious fairness considerations, the minimax framework is still well-motivated
 in any setting in which there are distinct ``contexts'' (the groups in a fairness scenario) for which we would like to minimize the worst-performing context.}
for both classification and regression. Here, we 
report our results on a  held-out validation set.  Our findings are consistent with our theoretical generalization guarantees discussed in Section \ref{subsec:generalization} --- that we obtain strong generalization performance when the data is of sufficient size.

\subsubsection{Regression: Finding Exact Solutions Efficiently}
\label{exact-regression}

The solution to a weighted linear regression is \textit{guaranteed} to be the regression function that minimizes the mean squared error across the dataset,  making linear regression a pure demonstration of our theory in Sections~\ref{sec:prelim} and \ref{sec:game}. We note that to solve weighted linear regressions \textit{efficiently}, sample weights must be non-negative, because negative weights make squared error non-convex. Our minimax and relaxed algorithms satisfy this property, giving us access to exact solutions in regression settings. In contrast, similar algorithms (e.g. those of \cite{agarwal2018reductions,Kearns}) for equalizing group error rates require negative sample weights, and cannot use linear regression for exact weighted error minimization, despite the convexity of the unweighted problem. This is another advantage of our approach.

\subsubsection{Classification: Non-convexity of 0/1 Loss}
\label{nonconvexity}

As opposed to mean squared error, 0/1 classification loss is non-convex. As a result, we cannot hope to  efficiently find solutions that exactly minimize classification loss in practice. Instead, we rely on convex surrogate loss functions such as log-loss --- the training objective for logistic regression --- which are designed to approximate classification loss. Note that lack of exact solutions for classification loss violates Assumption \ref{oracle assumption}, so the theoretical guarantees of Sections 2 and 3 may fail to hold. Our algorithm should be viewed as a principled heuristic in these settings. Note, when using logistic regression with \mmFair, we update sample weights based on the log-loss, so our algorithm can be viewed as provably solving an optimization problem with respect to log-loss; but we report errors in our plots with respect to 0/1 loss.\footnote{For additional intuition, in Section~\ref{sec:nonconvex} of the Appendix, we provide a visual comparing the loss surfaces of log-loss and 0/1 loss.}

\subsubsection{Paired Regression Classifier}
\label{prc}

In addition to logistic regression, we experiment with the paired regression classifier (PRC) used in \cite{agarwal2018reductions} and defined below. A paired regression classifier has the important feature that it only needs to solve a convex optimization problem \emph{even in the presence of negative sample weights}. This property is necessary for use with the equality of error rates algorithm from \cite{agarwal2018reductions}, because it generates negative sample weights.  We use this to benchmark our minimax algorithm. 

\begin{definition}[Paired Regression Classifier]
\label{def:PRC}
The paired regression classifier operates as follows: We form two weight vectors, $z^0$ and $z^1$, where $z^k_i$ corresponds to the penalty assigned to sample $i$ in the event that it is labeled $k$. For the correct labeling of $x_i$, the penalty is $0$. For the incorrect labeling, the penalty is the current sample weight of the point, $w_i$. We fit two linear regression models $h^0$ and $h^1$ to predict $z^0$ and $z^1$, respectively, on all samples. Then, given a new point $x$, we calculate $h^0(x)$ and $h^1(x)$ and output $h(x) = \text{argmin}_{k\in\{0,1\}} h^k(x)$.
\end{definition} 

We observe that the paired regression classifier produces a \emph{linear threshold function}, just as logistic regression does.

\subsubsection{Relaxation Methodology}
\label{sssec:relaxation_methodology}

We use \mmFair\ and \mmFairCon\ to explore and plot the tradeoff between minimax fairness and population accuracy as follows:
\begin{itemize}
    \item First we run \mmFair\ to determine the values of two quantities $\gamma_{\text{min}}$, the maximum group error of the randomized minimax model, and $\gamma_{\text{max}}$, the maximum group error of the population error minimizing model. The former is the minimum feasible value for $\gamma$, and the latter is the largest value for $\gamma$ we would ever want to accept.
    \item Next, we run \mmFairCon\ over a selection of gamma values in  $[\gamma_{\text{min}}, \gamma_{\text{max}}]$ to trace the achievable tradeoff between fairness and accuracy and plot the associated Pareto curve.
    % \item For each run of \mmFairCon\ above, we extract the trajectory, with the performance of the final randomized model denoted by a yellow dot and the associated value of gamma
    % \item Finally we overlay these trajectories are overlaid in a single plot and trace the Pareto curve represented linear combinations of these  models
\end{itemize}

In the false positive/negative case, we skip running \mmFair\ and directly use \mmFairCon\ over a range of $\gamma$ values above 0. This is because a minimax solution for false positive (or negative) group errors will always be zero, as witnessed by a constant classifier that predicts `False' (or `True') on all inputs. 
Thus we set $\gamma_{\text{min}}$ to 0.

Since each of our algorithms produces solutions in the form of randomized models with error reported in expectation, every linear combination of these models represents another randomized model. Further, it means that every point on our Pareto curve for population error vs. maximum group error --- including those falling on the line between two models --- can be achieved by some  randomized model in our class.

\subsection{Linear Regression Experiments}

As explained above, linear regression is an exact setting for demonstrating the properties of our algorithms, as we can solve weighted linear regression problems exactly and efficiently in practice. Our first results, shown in Figure~\ref{fig:simple_regression}, illustrate the behavior of \mmFair\ on the \textit{Communities} dataset. The left and center plots illustrate the errors and weights of the various groups, and the right plot denotes the ``trajectory" of our randomized model over time, showing how we update from our initial population-error minimizing model (labeled by a large pink dot) to our final randomized model achieving minimax group fairness (colored yellow).

Looking at the group error plot, we see that the plurality black communities --- denoted by an orange line --- begin with the largest error of any group with an MSE of ~0.3, while the plurality white communities --- denoted by a blue line --- have the lowest MSE at a value ~0.005. Turning to the group weights plot, we see that our algorithm responds to this disparity by up-weighting samples representing plurality black communities and by down-weighting those representing plurality white communities.  Inspecting the trajectory plot, we note that the algorithm gradually decreases maximum group error at the cost of increased population error, with intermediate models tracing a convex tradeoff curve between these two types of error. After many rounds, an approximate minimax solution for group-fairness is reached, achieving a maximum group error value slightly under 0.02. We observe that our minimax solution nearly equalizes errors on three of the groups, with one group (white communities) having error slightly below the minimax value. 

Near equalization of the highest group errors in minimax solutions as seen in this example is frequent and well-explained. In any error optimal solution, the only way to decrease the error of one group is to increase the error of another. Hence, whenever the loss landscape is continuous over our class of models (as it is in our case, because we allow for distributions over classifiers), minimax optimal solutions will generally equalize the error of at least two groups. But as we see in our examples, it does not require equalizing error across all groups when there are more than two. 

\begin{figure*}[h]
  \centering
\begin{subfigure}{.32\textwidth}
\includegraphics[width=\textwidth]{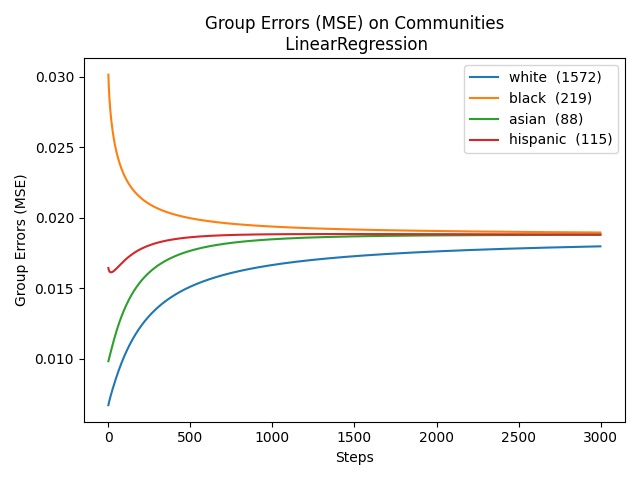}
\end{subfigure}
\begin{subfigure}{.32\textwidth}
\includegraphics[width=\textwidth]{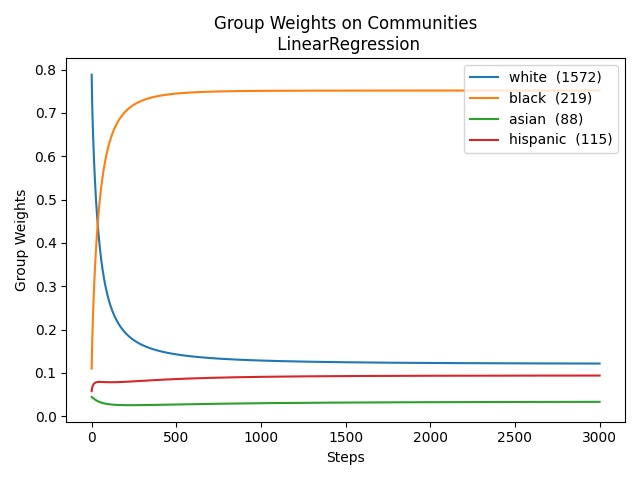}
\end{subfigure}
\begin{subfigure}{.32\textwidth}
\includegraphics[width=\textwidth]{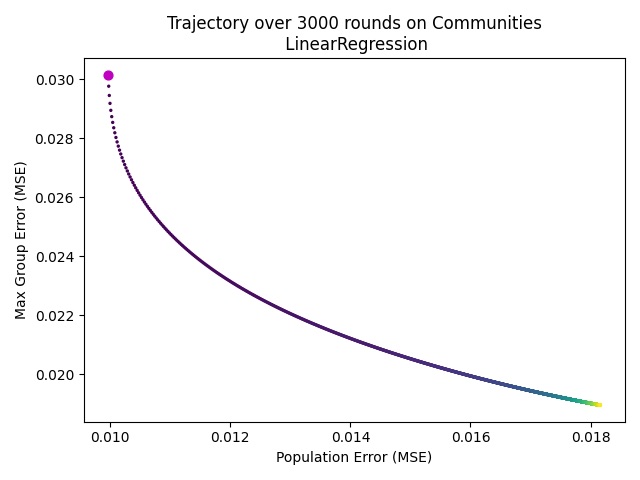}
\end{subfigure}
\caption{\mmFair\ on Communities Dataset}
\label{fig:simple_regression}
\end{figure*}

\subsubsection{Comparing Minimax to Equality}

Next, we provide a comparison between our minimax algorithm and the equal error rates formulation of \cite{agarwal2018reductions}. We note that, while linear regression is an excellent fit for our minimax algorithm, it poses difficulties in the equal error rates framework. In particular, similar primal/dual algorithms for equalizing error rates across groups require the use of negative sample weights, because the dual solution to a linear program with equality constraints generically requires negative variables. Negative sample weights destroy convexity for objective functions like squared error that are convex with non-negative weights. For this reason, we can only use the equal-error algorithm of \cite{agarwal2018reductions} in a meaningful way for linear regression in settings in which the sample weights (by luck) never become negative. On the Bike dataset, we meet this condition, and are therefore able to provide a meaningful comparison between the solutions produced by the two algorithms which is illustrated in Figure~\ref{fig:bike_comp}. We observe that the only difference between the two solutions, is that error in Winter and Autumn increases to the minimax value when we move from minimax to equality. This highlights an important point: enforcing equality of group errors may significantly hinder our performance on members of low-error groups, without providing benefit to those of higher-error groups. Though the bike dataset itself is not naturally fairness-sensitive, the properties illustrated in this example can occur in any  dataset.
\begin{figure*}
\centering
\includegraphics[width=0.48\textwidth]{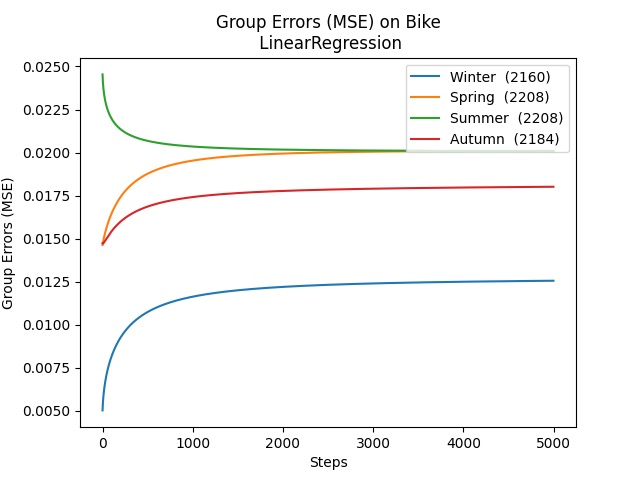}
\includegraphics[width=0.48\textwidth]{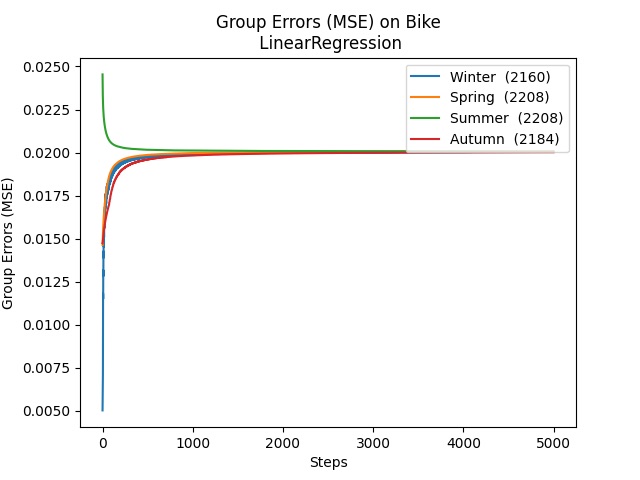}
\captionsetup[subfigure]{labelformat=empty}
\caption{Comparing Minimax (left) and Equal Error (right) Models on Bike Dataset}
\label{fig:bike_comp}
\end{figure*}

\subsubsection{Relaxing Fairness Constraints}

We also investigate the ``cost'' of fairness with respect to overall accuracy by using \mmFair\ and \mmFairCon\ as described in Section~\ref{sssec:relaxation_methodology}. In particular, for each run of \mmFairCon\, we extract the trajectory plot and mark the performance of the final randomized model with a yellow dot labeled with the associated value of gamma. In Figure~\ref{fig:communities_trajectories}, we perform this procedure on the Communities dataset, and overlay these trajectories onto a single plot. We then connect the endpoints of these trajectories to trace a Pareto curve (denoted in red, dashed line) that represents the tradeoff between fairness and accuracy. We observe that the relationship between expected population error and maximum group error is decreasing and convex (both within trajectories and between their endpoints) illustrating a clear tradeoff between the two objectives.  
\begin{figure}
\centering
\includegraphics[width=0.43\textwidth]{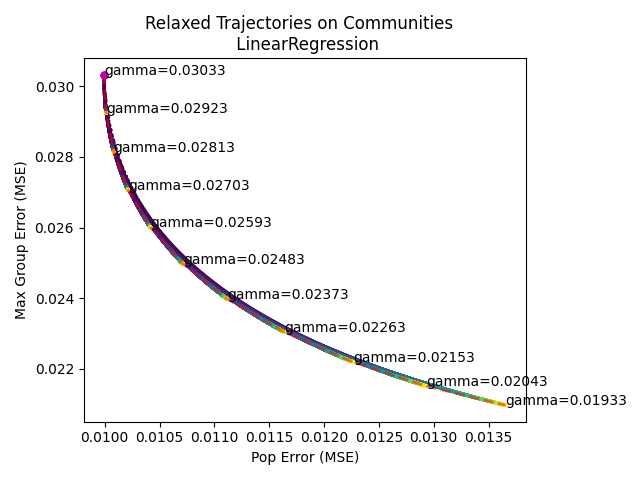}
\captionsetup[subfigure]{labelformat=empty}
\caption{Fairness/Accuracy Tradeoff on Communities}

\label{fig:communities_trajectories}
\end{figure}
\subsection{Classification Experiments}
\subsubsection{Comparing Minimax to Equality}

Our first experiment reveals significant advantages of our algorithm for minimax fairness as compared to the  equal error rates algorithm of \cite{agarwal2018reductions} on the \textit{COMPAS} dataset. When running the equal error algorithm, we are restricted to using the  PRC classifier,  because of its ability to handle negative sample weights. For the minimax algorithm, which utilizes only non-negative sample weights, we are free to use potentially better base algorithms like logistic regression. Nevertheless,  we use only PRC  in this experiment to provide a more direct comparison. 

In Figure \ref{fig:COMPAS} we compare the performance of  \mmFair\ to the equal-error algorithm of \cite{agarwal2018reductions} on the \textit{COMPAS }dataset, when enforcing group fairness for both race and gender groups simultaneously. We observe that our minimax solution strongly Pareto dominates the equality of errors solution, as the equality constraints not only inflate errors on the lower error groups but also on the groups that start with the highest error. In particular, the maximum group error for the minimax model is approximately 0.325 (with other group errors as low as 0.295), while the error equalizing model forces the errors of all groups to be (approximately) equal values in the range 0.38-0.39. The explanation for the poor performance of the equality of errors solution is straightforward. Once the highest error groups have their error decreased to the minimax value, the only way to achieve equality of errors is by increasing the error on the lower error groups. It happened that the only way to achieve this was by making a model that performed worse on \textit{all} of the groups.

\begin{figure*}%[h]
\centering
\hspace{0.0cm}
\includegraphics[width = 0.44\textwidth]{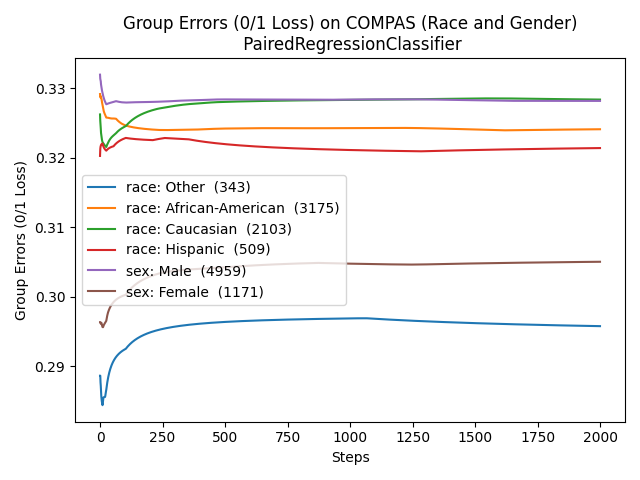}
\hspace{0.0cm}
\includegraphics[width = 0.44\textwidth]{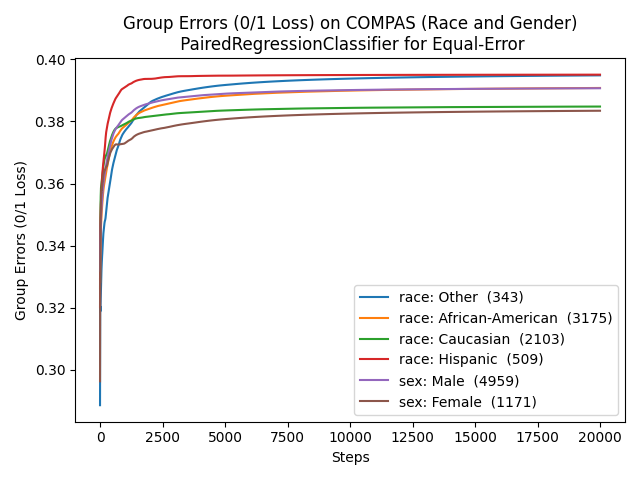}
\captionsetup[subfigure]{labelformat=empty}
\caption{Minimax (left) vs. Equal Errors (right) on COMPAS Dataset with PRC}
\label{fig:COMPAS}
\end{figure*}

\subsubsection{Relaxation and Pareto Curves}

With the potential issues of non-convexity in mind, we move to an experimental analysis of overall error versus max group error tradeoff curves produced by \mmFairCon\ in the classification setting. In Figure~\ref{fig:marketing}, we predict whether or not an individual will subscribe to a term deposit in a Portuguese bank using the Marketing dataset of \cite{Moro, Dua:2019}. In this experiment, we train on log-loss, using it as the error metric for the updates of both the learner and regulator. As the theory dictates, by convexity of the log-loss, we observe an excellent convex tradeoff curve across different values of $\gamma$ when measured by log-loss (shown in the left plot). When we examine the corresponding tradeoff curve with respect to classification loss --- shown in the right plot --- we see that the shape is similar. This indicates that, for this dataset, log-loss is a good surrogate for 0/1 loss, and our fairness guarantees may be realized in practice.
\begin{figure*}[h]
\centering
\includegraphics[width=0.44\textwidth]{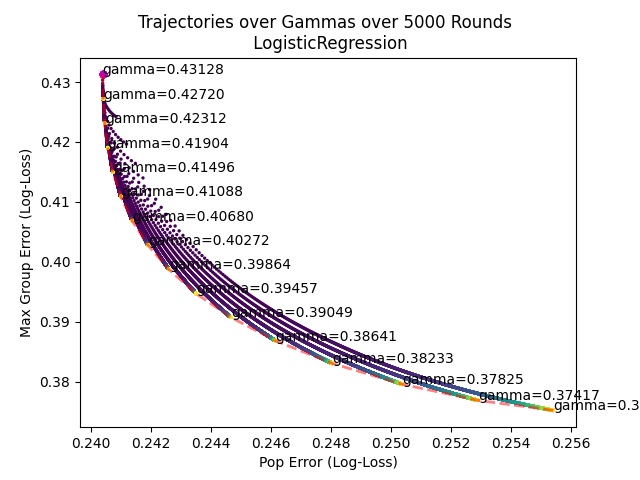}
\includegraphics[width=0.44\textwidth]{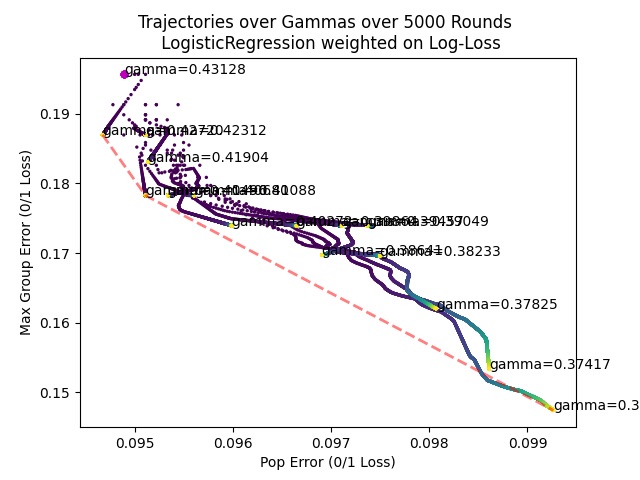}
\captionsetup[subfigure]{labelformat=empty}
\caption{Fairness/Accuracy Tradeoff on Marketing Dataset for Log-Loss (left) and Classification Loss (right)}
\label{fig:marketing}
\end{figure*}

\subsubsection{False Positive and False Negative Rates}

In some contexts, we may care more about false positive or false negative rates than overall error rates. For example, when predicting criminal recidivism (using, for example, the popular \textit{COMPAS} dataset), we may consider a false positive to be more harmful to an individual than a false negative, as it can lead to mistaken incarceration. As mentioned previously, one can trivially achieve a false positive rate of zero using a constant classifier that blindly makes negative predictions on all inputs (minimizing false negatives is analogous). Therefore, achieving minimax group false positive or false negative rates alone is uninteresting. Instead, we care about the tradeoff between overall error and maximum group false positive or false negative rates.

In Figure~\ref{fig:falsePos} we examine this tradeoff on the \textit{COMPAS} dataset with respect to false positive rates. When population error is minimized (indicated with a pink dot), the maximum false positive rate of any group is around ~0.37. As we decrease $\gamma$, we see a gradual increase in population error tracing a nearly-convex tradeoff curve. When the maximum group false positive rate reaches 0, the population error achieved by our algorithm is ~0.44. Importantly, we are not restricted to picking from these extremal points on the curve. Suppose we believed that a maximum group false positive rate of 0.05 was acceptable. Our tradeoff curve shows that we can meet this constraint while achieving a population error of only 0.39.

\begin{figure}
\centering
\includegraphics[width=0.44\textwidth]{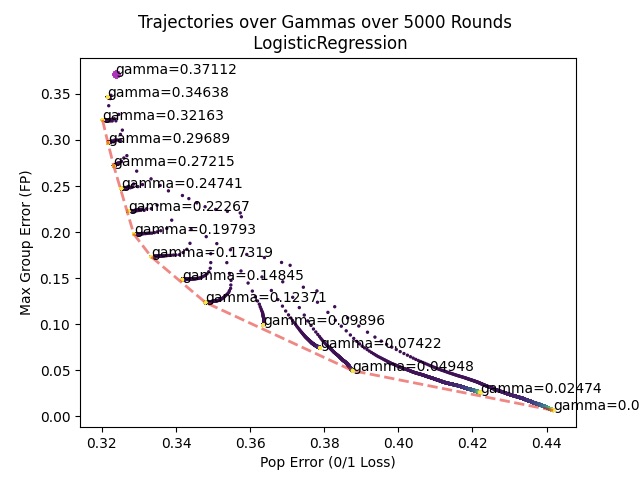}
\captionsetup[subfigure]{labelformat=empty}
\caption{Fairness/Accuracy Tradeoff for FP on COMPAS}
\label{fig:falsePos}
\end{figure}

\subsection{Demonstrating Generalization}
\label{subsec:generalization-experiments}
 In this section, we demonstrate the generalization performance of our minimax algorithm on large, real datasets for both classification and regression. 

\subsubsection{Regression}
In Figure \ref{fig:bike-generalization} we investigate generalization performance of the minimax algorithm in the regression setting using the \textit{Bike} dataset with 25\% of the data withheld for validation. We observe that for all four seasonal groups, the in-sample and out-of-sample errors are nearly identical for both the final minimax model as well as the intermediate randomized models produced by the algorithm. This result serves as a demonstration of the uniform convergence of group errors predicted by the theory in section \ref{subsec:generalization}, and shows that our algorithm is capable of strong generalization even when groups only have a few thousand instances each.
\begin{figure}[h]
\centering
\includegraphics[width=0.44\textwidth]{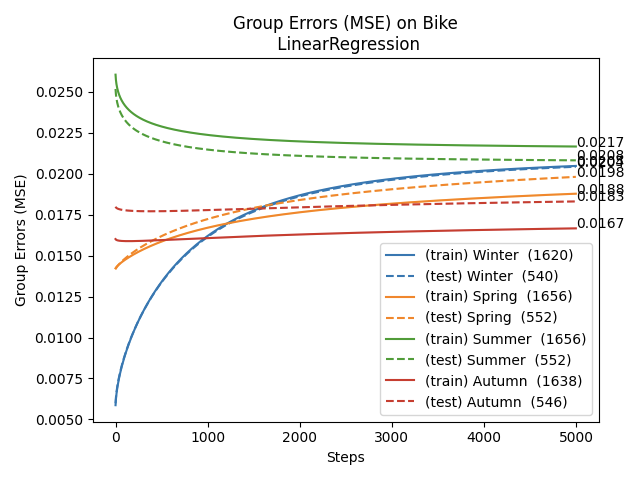}
\caption{Train vs. Test Performance for Minimax on Bike}
\label{fig:bike-generalization}
\end{figure}
 \subsubsection{Classification}
In Figure \ref{fig:kdd-generalization} we examine the generalization performance of the minimax algorithm in the classification setting using the \textit{Internet Traffic} dataset with 25\% of the data withheld for validation. We observe strong generalization with respect to the maximum group error objective. Note the scale of the y-axis: despite the wide separation of the curves, the the out-of-sample maximum group error exceeds the in-sample error by only 0.0007. In particular, we observe that the largest group error in-sample is 0.0022 (on the tcp group) while the largest group error out of sample is 0.0029 (on the icmp group). Moreover, we observed that the differences between in and out-of-sample errors for each group were not only small (with the largest difference having value only 0.0013) but also tended to decrease over the course of the algorithm. That is, the difference between the in and out-of-sample errors for the unweighted model were generally larger than the corresponding differences in for the minimax model. This indicates that any lack of generalisation is a generic property of linear classifiers on the dataset rather than specific of our minimax algorithm. 

\begin{figure}[!h]
\centering
	\includegraphics[width=0.44\textwidth]{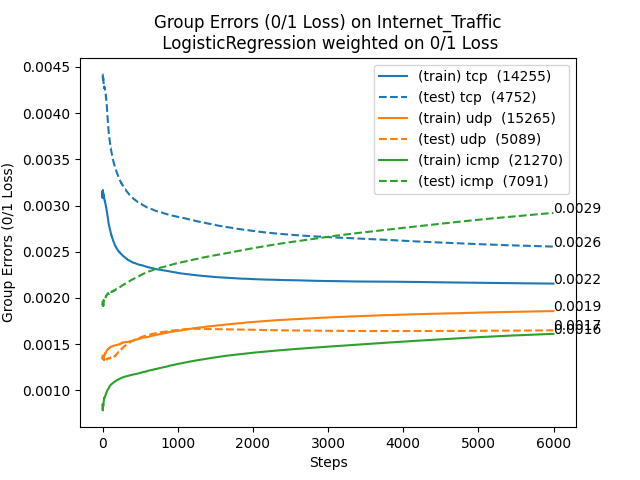}
\captionsetup[subfigure]{labelformat=empty}
\caption{Train vs. Test Performance on Internet Traffic}
\label{fig:kdd-generalization}
\end{figure}

\section{Conclusion}
We have provided a provably convergent, practical algorithm for solving minimax group fairness problems and error minimization problems subject to group-dependent error upper bound constraints. Our meta-algorithm supports any statistical model class making use of sample weights and our methods inherit the generalization guarantees of our base class of models.  As our theory suggests, performance is excellent when we can exactly solve weighted empirical risk minimization problems (as is the case with linear regression and other problems with convex objectives). We provide a thorough discussion of instances in which this assumption is infeasible, notably when a convex surrogate loss function is used in place of 0/1 classification loss. In these instances, our algorithm must be viewed as a principled heuristic. Finally, in regression and classification settings, we demonstrate that our algorithms for optimizing minimax group error results in overall error that is no worse than the error that can be obtained when attempting to equalize error across groups and can be markedly better for some groups compared to an equal error solution. In high stakes settings, this Pareto improvement may be highly desirable, in that it avoids harming any group more than is necessary to reduce the error of the highest error group.

\newpage

\bibliographystyle{plainnat}
\bibliography{arxiv_aies_main}

\begin{thebibliography}{33}
\providecommand{\natexlab}[1]{#1}
\providecommand{\url}[1]{\texttt{#1}}
\expandafter\ifx\csname urlstyle\endcsname\relax
  \providecommand{\doi}[1]{doi: #1}\else
  \providecommand{\doi}{doi: \begingroup \urlstyle{rm}\Url}\fi

\bibitem[Agarwal et~al.(2018)Agarwal, Beygelzimer, Dudík, Langford, and
  Wallach]{agarwal2018reductions}
Alekh Agarwal, Alina Beygelzimer, Miroslav Dudík, John Langford, and Hanna
  Wallach.
\newblock A reductions approach to fair classification.
\newblock In \emph{Proceedings of the 35th International Conference on Machine
  Learning}, 2018.

\bibitem[Agarwal et~al.(2019)Agarwal, Dudik, and Wu]{agarwal2019fair}
Alekh Agarwal, Miroslav Dudik, and Zhiwei~Steven Wu.
\newblock Fair regression: Quantitative definitions and reduction-based
  algorithms.
\newblock In Kamalika Chaudhuri and Ruslan Salakhutdinov, editors,
  \emph{Proceedings of Machine Learning Research}, volume~97, pages 120--129,
  Long Beach, California, USA, 09--15 Jun 2019. PMLR.
\newblock URL \url{http://proceedings.mlr.press/v97/agarwal19d.html}.

\bibitem[Asadpour and Saberi(2010)]{fairdivision}
Arash Asadpour and Amin Saberi.
\newblock An approximation algorithm for max-min fair allocation of indivisible
  goods.
\newblock \emph{SIAM Journal on Computing}, 39\penalty0 (7):\penalty0
  2970--2989, 2010.

\bibitem[Bay et~al.(2000)Bay, Kibler, Pazzani, and Smyth]{kdd}
Stephen~D. Bay, Dennis~F. Kibler, Michael~J. Pazzani, and Padhraic Smyth.
\newblock The uci kdd archive of large data sets for data mining research and
  experimentation, 2000.

\bibitem[Cesa-Bianchi and Lugosi(2006)]{cesa-bianchi_lugosi_2006}
Nicolo Cesa-Bianchi and Gabor Lugosi.
\newblock \emph{Prediction, Learning, and Games}.
\newblock Cambridge University Press, 2006.
\newblock \doi{10.1017/CBO9780511546921}.

\bibitem[Cesa-Bianchi and Lugosi(1999)]{cesa-bianchi1999}
Nicolò Cesa-Bianchi and Gábor Lugosi.
\newblock On prediction of individual sequences.
\newblock \emph{Ann. Statist.}, 27\penalty0 (6):\penalty0 1865--1895, 12 1999.
\newblock \doi{10.1214/aos/1017939242}.
\newblock URL \url{https://doi.org/10.1214/aos/1017939242}.

\bibitem[Chen et~al.(2017)Chen, Lucier, Singer, and Syrgkanis]{chen}
Robert~S. Chen, Brendan Lucier, Yaron Singer, and Vasilis Syrgkanis.
\newblock Robust optimization for non-convex objectives.
\newblock In I.~Guyon, U.~V. Luxburg, S.~Bengio, H.~Wallach, R.~Fergus,
  S.~Vishwanathan, and R.~Garnett, editors, \emph{Advances in Neural
  Information Processing Systems}, volume~30, pages 4705--4714. Curran
  Associates, Inc., 2017.
\newblock URL
  \url{https://proceedings.neurips.cc/paper/2017/file/10c66082c124f8afe3df4886f5e516e0-Paper.pdf}.

\bibitem[Chouldechova et~al.(2018)Chouldechova, Benavides-Prado, Fialko, and
  Vaithianathan]{chouldechova2018case}
Alexandra Chouldechova, Diana Benavides-Prado, Oleksandr Fialko, and Rhema
  Vaithianathan.
\newblock A case study of algorithm-assisted decision making in child
  maltreatment hotline screening decisions.
\newblock In \emph{Conference on Fairness, Accountability and Transparency},
  pages 134--148, 2018.

\bibitem[Cortez and Silva(2008)]{Cortez2008UsingDM}
Paulo Cortez and Alice Silva.
\newblock Using data mining to predict secondary school student performance.
\newblock \emph{EUROSIS}, 01 2008.

\bibitem[Cotter et~al.(2019)Cotter, Jiang, and Sridharan]{cotter}
Andrew Cotter, Heinrich Jiang, and Karthik Sridharan.
\newblock Two-player games for efficient non-convex constrained optimization.
\newblock In Aurélien Garivier and Satyen Kale, editors, \emph{Proceedings of
  the 30th International Conference on Algorithmic Learning Theory}, volume~98
  of \emph{Proceedings of Machine Learning Research}, pages 300--332, Chicago,
  Illinois, 22--24 Mar 2019. PMLR.
\newblock URL \url{http://proceedings.mlr.press/v98/cotter19a.html}.

\bibitem[Dua and Graff(2017)]{Dua:2019}
Dheeru Dua and Casey Graff.
\newblock {UCI} machine learning repository, 2017.
\newblock URL \url{http://archive.ics.uci.edu/ml}.

\bibitem[Dwork et~al.(2018)Dwork, Immorlica, Kalai, and Leiserson]{decoupled1}
Cynthia Dwork, Nicole Immorlica, Adam~Tauman Kalai, and Max Leiserson.
\newblock Decoupled classifiers for group-fair and efficient machine learning.
\newblock In \emph{Conference on Fairness, Accountability and Transparency},
  pages 119--133, 2018.

\bibitem[E and Cho(2020)]{doi:10.1080/22797254.2020.1725789}
Sathishkumar~V E and Yongyun Cho.
\newblock A rule-based model for {Seoul} bike sharing demand prediction using
  weather data.
\newblock \emph{European Journal of Remote Sensing}, pages 1--18, 2020.
\newblock \doi{10.1080/22797254.2020.1725789}.

\bibitem[E et~al.(2020)E, Park, and Cho]{E2020353}
Sathishkumar~V E, Jangwoo Park, and Yongyun Cho.
\newblock Using data mining techniques for bike sharing demand prediction in
  metropolitan city.
\newblock \emph{Computer Communications}, 153:\penalty0 353 -- 366, 2020.
\newblock ISSN 0140-3664.
\newblock \doi{https://doi.org/10.1016/j.comcom.2020.02.007}.
\newblock URL
  \url{http://www.sciencedirect.com/science/article/pii/S0140366419318997}.

\bibitem[Freund and Schapire(1996)]{Freund}
Yoav Freund and Robert~E. Schapire.
\newblock Game theory, on-line prediction and boosting.
\newblock In \emph{Proceedings of the Ninth Annual Conference on Computational
  Learning Theory}, 1996.

\bibitem[Freund and Schapire(1997)]{10.1006/jcss.1997.1504}
Yoav Freund and Robert~E Schapire.
\newblock A decision-theoretic generalization of on-line learning and an
  application to boosting.
\newblock \emph{J. Comput. Syst. Sci.}, 55\penalty0 (1):\penalty0 119–139,
  August 1997.
\newblock ISSN 0022-0000.
\newblock \doi{10.1006/jcss.1997.1504}.
\newblock URL \url{https://doi.org/10.1006/jcss.1997.1504}.

\bibitem[Hahne(1991)]{scheduling}
Ellen~L. Hahne.
\newblock Round-robin scheduling for max-min fairness in data networks.
\newblock \emph{IEEE Journal on Selected Areas in communications}, 9\penalty0
  (7):\penalty0 1024--1039, 1991.

\bibitem[Kearns et~al.(2018)Kearns, Neel, Roth, and Wu]{Kearns}
Michael Kearns, Seth Neel, Aaron Roth, and Zhiwei~Steven Wu.
\newblock Preventing fairness gerrymandering: Auditing and learning for
  subgroup fairness.
\newblock In \emph{Proceedings of the 35th International Conference on Machine
  Learning}. Stockholm, Sweden, PMLR 80, 2018.

\bibitem[Kearns et~al.(2019)Kearns, Neel, Roth, and Wu]{subgroup2}
Michael Kearns, Seth Neel, Aaron Roth, and Zhiwei~Steven Wu.
\newblock An empirical study of rich subgroup fairness for machine learning.
\newblock In \emph{Proceedings of the Conference on Fairness, Accountability,
  and Transparency}, pages 100--109, 2019.

\bibitem[Kearns and Vazirani(1994)]{kearnsVazirani}
Michael~J. Kearns and Umesh~V. Vazirani.
\newblock \emph{An Introduction to Computational Learning Theory}.
\newblock MIT Press, Cambridge, MA, USA, 1994.
\newblock ISBN 0262111934.

\bibitem[Lahoti et~al.(2020)Lahoti, Beutel, Chen, Lee, Prost, Thain, Wang, and
  Chi]{lahoti2020fairness}
Preethi Lahoti, Alex Beutel, Jilin Chen, Kang Lee, Flavien Prost, Nithum Thain,
  Xuezhi Wang, and Ed~H. Chi.
\newblock Fairness without demographics through adversarially reweighted
  learning.
\newblock In \emph{Advances in Neural Information Processing Systems}, 2020.

\bibitem[Martinez et~al.(2020)Martinez, Bertran, and Sapiro]{Martinez}
Natalie Martinez, Martin Bertran, and Guillermo Sapiro.
\newblock Minimax {Pareto} fairness: A multi objective perspective.
\newblock In \emph{Proceedings of the 37th International Conference on Machine
  Learning}. Vienna, Austria, PMLR 119, 2020.

\bibitem[Moro et~al.(2014)Moro, Cortez, and Rita]{Moro}
S.~Moro, P.~Cortez, and P.~Rita.
\newblock A data-driven approach to predict the success of bank telemarketing.
\newblock \emph{Decis. Support Syst.}, 62:\penalty0 22--31, 2014.

\bibitem[ProPublica(2020)]{compas}
ProPublica.
\newblock \emph{COMPAS Recidivism Risk Score Data and Analysis}.
\newblock Broward County Clerk’s Office, Broward County Sherrif's Office,
  Florida Department of Corrections, ProPublica, September 2020.
\newblock URL
  \url{https://www.propublica.org/datastore/dataset/compas-recidivism-risk-score-data-and-analysis}.

\bibitem[Redmond and Baveja(2002)]{REDMOND2002660}
Michael Redmond and Alok Baveja.
\newblock A data-driven software tool for enabling cooperative information
  sharing among police departments, 2002.
\newblock ISSN 0377-2217.
\newblock URL
  \url{http://www.sciencedirect.com/science/article/pii/S0377221701002648}.

\bibitem[Samadi et~al.(2018)Samadi, Tantipongpipat, Morgenstern, Singh, and
  Vempala]{samadi2018price}
Samira Samadi, Uthaipon Tantipongpipat, Jamie~H Morgenstern, Mohit Singh, and
  Santosh Vempala.
\newblock The price of fair {PCA}: One extra dimension.
\newblock In \emph{Advances in Neural Information Processing Systems}, pages
  10976--10987, 2018.

\bibitem[U.~S. Department~of Commerce(1900)]{census1}
Bureau of the~Census U.~S. Department~of Commerce.
\newblock Census of population and housing 1990 {United States}: Summary tape
  file 1a \& 3a (computer files), 1900.

\bibitem[U.S. Department Of~Commerce(1992)]{census2}
Bureau Of The Census~Producer U.S. Department Of~Commerce, 1992.

\bibitem[U.S. Department~of Justice(1992)]{census3}
Bureau of Justice~Statistics U.S. Department~of Justice, 1992.

\bibitem[U.S. Department~of Justice(1995)]{census4}
Federal Bureau of~Investigation U.S. Department~of Justice.
\newblock Crime in the {United States} (computer file), 1995.

\bibitem[Ustun et~al.(2019)Ustun, Liu, and Parkes]{decoupled2}
Berk Ustun, Yang Liu, and David Parkes.
\newblock Fairness without harm: Decoupled classifiers with preference
  guarantees.
\newblock In \emph{International Conference on Machine Learning}, pages
  6373--6382, 2019.

\bibitem[Vapnik and Chervonenkis(1971)]{vapnik}
V.~N. Vapnik and A.Y. Chervonenkis.
\newblock Chervonenkis: On the uniform convergence of relative frequencies of
  events to their probabilities.
\newblock \emph{Theory of Probability and its Applications}, 16, 1971.
\newblock URL \url{https://doi.org/10.1137/1116025}.

\bibitem[Zinkevich(2003)]{zinkevich}
Martin Zinkevich.
\newblock Online convex programming and generalized infinitesimal gradient
  ascent.
\newblock In \emph{Proceedings of the Twentieth International Conference on
  Machine Learning}. Washington, DC, 2003.

\end{thebibliography}

\newpage
\appendix

\section{Update Rules in \mmFairCon\ for Overlapping Groups}
\label{sec:appendix}

In the original presentation of \mmFairCon, we assumed that the groups were disjoint for ease of presentation. Here we provide a more general derivation of the update rules for the learner and regulator when groups are allowed to be overlapping. The settings and notation of Section~\ref{sec:prelim} apply, but now the groups $\{G_k\}_{k=1}^{K}$ are not necessarily disjoint. We want to minimize population error while bounding the error of each group by $\gamma$. So, our constrained optimization problem is:
\begin{mini}
    {h \in H}{\epsilon(h)}{}{}
    \addConstraint{\epsilon_k(h)}{\leq \gamma,\:}{ k=1,\ldots,K}
\end{mini}

Then, the corresponding Lagrangian dual is:
\begin{align*}
    &F(\lambda,h) = \epsilon(h) + \sum_{k} \lambda_k (\epsilon_k(h) - \gamma)\\
    &=\frac{1}{n}\sum_i L(h(x_i),y_i) \\
    &+\sum_k \lambda_k \left(\frac{1}{|G_k|}\sum_i L(h(x_i),y_i)\mathbb{I}\{(x_i,y_i) \in G_k\} - \gamma\right) \\
    &= \frac{1}{n} \sum_i L(h(x_i),y_i) \\
    &+ \sum_i \left( \sum_k \frac{\lambda_k}{|G_k|} L(h(x_i),y_i) \mathbb{I}\{(x_i,y_i) \in G_k\}\right) - \sum_k \lambda_k \gamma \\
    &= \sum_i L(h(x_i),y_i)\left(\frac{1}{n}+ \sum_k \frac{\lambda_j}{|G_k|}\mathbb{I}\{(x_i,y_i) \in G_k\}\right) \\
    &- \sum_k \lambda_k \gamma \\
    &= \sum_i w_i L(h(x_i),y_i) - \sum_k \lambda_k \gamma 
\end{align*}

Where $w_i = \frac{1}{n} + \sum_k \frac{\lambda_k}{|G_k|}\mathbb{I}\{(x_i,y_i) \in G_k\}$ and $$\frac{\delta F}{\delta \lambda_k} = \frac{1}{|G_k|} \sum_i \mathbb{I}\{(x_i,y_i) \in G_k\} L(h(x_i),y_i)) - \gamma$$

We can then use the sample weights $w$ in the learner's step of the constrained optimization problem. In particular, at each round $t$, the learner will find $h_t := \text{argmin}_{h \in H} \sum_i w_i \cdot L(h(x_i),y_i)$
and the regulator will update the $\lambda$ vector with
$$\lambda_k := \left(\lambda_k + \eta_t \cdot \left(\frac{1}{|G_k|} \sum_i \mathbb{I}\{(x_i,y_i) \in G_k\} L(h(x_i),y_i) - \gamma\right)\right)_+$$

\newpage
\section{Proof and Discussion of Generalization Theorem}
Here, for completeness, we provide a brief proof for Theorem~\ref{thm:generalization}.

\begin{proof} Fix any group $i$.
A standard uniform convergence argument tells us that with probability $1-\delta$ over the $n_i$ samples from group $i$, for every (deterministic) $h_j \in H$, the generalization gap is of order $\epsilon_i = O\left(\sqrt{\frac{\log{\frac{1}{\delta}} + d \log n_i}{n_i}}\right)$\cite{vapnik, kearnsVazirani}:
 
\begin{align*}
|\epsilon_i(h_j) - \epsilon'_i(h_j) | \leq \epsilon_i = O\left(\sqrt{\frac{\log{\frac{1}{\delta}} + d \log n_i}{n_i}}\right)
\end{align*}

Now consider the randomized model $\tilde{h}:= \sum_j \alpha_j h_j$ with weights $\alpha_j$ on deterministic models $h_j \in H$ such that $\sum_j \alpha_j = 1$. Note that $\tilde{h} \in H$ by convexity of $H$. Then, the average error of $\tilde{h}$ on group $i$ will be the weighted sum of model errors on group $i$ both in and out of sample:
\begin{align*}
\begin{split}
\epsilon_i(\tilde{h})=\sum_j \alpha_j \epsilon_i(h_j)
\end{split}
\begin{split}
\epsilon'_i(\tilde{h})=\sum_j \alpha_j \epsilon'_i(h_j)
\end{split}
\end{align*}
Finally, we bound each expected difference pointwise by $\epsilon_i$. With probability $1-\delta$,
\begin{align*}
&|\epsilon_i(\tilde{h}) - \epsilon'_i(\tilde{h}) | 
=\sum_j \alpha_j |\epsilon_i(h_j) - \epsilon'_i(h_j)|\\
&= \sum_j \alpha_j \epsilon_i \leq \epsilon_i =  O\left(\sqrt{\frac{\log{\frac{1}{\delta}} + d \log n_i}{n_i}}\right)
\end{align*}

\end{proof}
\label{sec:generalization_appendix}

\newpage
\section{Convex and Non-Convex Surface Example}
In Figure~\ref{fig:nonconvex}, we compare a log-loss surface with the corresponding 0/1 loss. In experiments, we optimize log-loss, but plot 0/1 loss. We are provably optimizing the log-loss, but our guarantees with respect to 0/1 loss must be evaluated empirically.

\begin{figure}[h]
\captionsetup[subfigure]{labelformat=empty}
\centering
\caption{Comparing Convexity of 0/1 Loss vs. Log-Loss}
\subfloat[Non-Convex Classification Loss Surface]{
\includegraphics[width=0.5\textwidth]{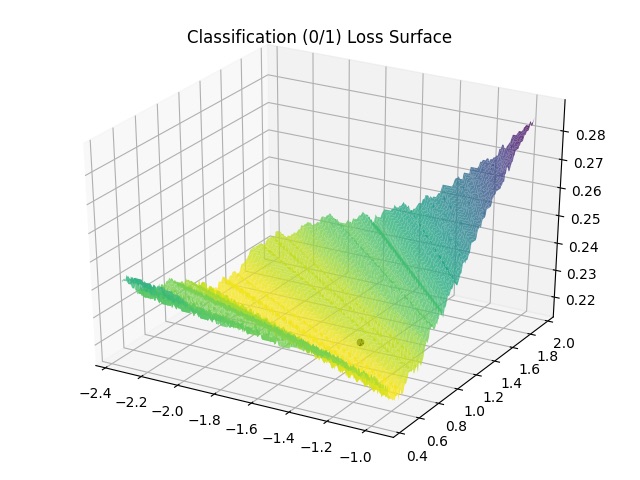}}
\hspace{0.0cm}
\subfloat[Convex Log-Loss Surface]{
\includegraphics[width=0.5\textwidth]{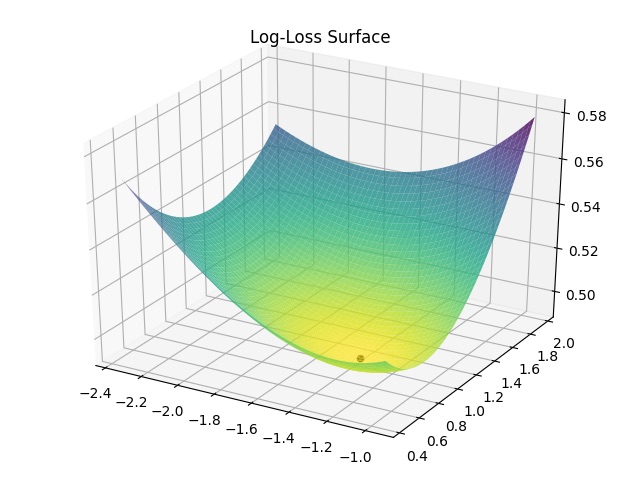}}
\label{fig:nonconvex}
\end{figure}

\label{sec:nonconvex}
\end{document}